\pdfoutput=1

\documentclass[11pt]{article}

\usepackage[preprint]{acl}

\usepackage{times}
\usepackage{latexsym}
\usepackage{amsmath}
\usepackage{amsfonts}
\usepackage{booktabs}
\usepackage{enumitem}
\usepackage{algorithm}
\usepackage{algpseudocode}

\usepackage{listings}
\usepackage{xcolor}

\lstdefinestyle{compactpython}{
    language=Python,                     
    basicstyle=\ttfamily\scriptsize,     
    keywordstyle=\color{blue},           
    commentstyle=\color{green!50!black}, 
    stringstyle=\color{red},             
    showstringspaces=false,              
    frame=none,                          
    breaklines=true,                     
    breakatwhitespace=true,              
    tabsize=1,                           
    xleftmargin=0em,                      
    morekeywords=[2]{critic,reward,lm, logsigmoid},        
    keywordstyle=[2]\color{red},      
    morekeywords=[3]{wd},        
    keywordstyle=[3]\color{magenta},      
}

\usepackage{multirow}
\usepackage{CJKutf8}

\usepackage{url}
\usepackage{booktabs}
\newcommand\allbold[1]{{\boldmath\textbf{#1}}}
\usepackage{csquotes}
\usepackage{graphicx}

\usepackage{xspace}
\definecolor{expert}{HTML}{008000}
\definecolor{error}{HTML}{f96565}
\definecolor{learner}{HTML}{F79646}
\definecolor{perfblue}{RGB}{64, 114, 175}
\definecolor{hgreen}{RGB}{217, 234, 211}
\definecolor{hred}{RGB}{244, 204, 204}
\definecolor{hgray}{RGB}{200, 200, 200}
\definecolor{lightgray}{gray}{0.9}
\definecolor{lightgreen}{rgb}{0.8,1,0.8}
\definecolor{perfblue}{RGB}{64, 114, 175}
\definecolor{perfred}{RGB}{220, 60, 60}
\newcommand{\cellgreen}{\cellcolor{hgreen}}
\newcommand{\method}{\texttt{DIAL}\xspace}

\usepackage[T1]{fontenc}

\usepackage[utf8]{inputenc}

\usepackage{microtype}

\usepackage{inconsolata}

\usepackage{graphicx}

\usepackage{amsmath}      
\usepackage{amssymb}      
\usepackage{amsthm}       
\usepackage{mathrsfs}     
\usepackage{graphicx}     
\usepackage{booktabs}     
\usepackage{bbm}          
\newtheorem{theorem}{Theorem}
\newtheorem{lemma}{Lemma}
\theoremstyle{remark}
\newtheorem*{proof*}{Proof}

\newcommand\numberthis{\addtocounter{equation}{1}\tag{\theequation}}

\newcommand{\figGap}[0]{\vspace{-\baselineskip}}

\title{Aligning LLMs with Domain Invariant Reward Models}

\author{David Wu \\
  \texttt{david9dragon9@gmail.com} \\\And
  Sanjiban Choudhury \\
  Cornell University \\
  \texttt{sc2582@cornell.edu} \\}

\begin{document}
\maketitle
\begin{abstract}
Aligning large language models (LLMs) to human preferences is challenging in domains where preference data is unavailable.
We address the problem of learning reward models for such target domains by leveraging feedback collected from simpler source domains, where human preferences are easier to obtain.
Our key insight is that, while domains may differ significantly, human preferences convey \emph{domain-agnostic} concepts that can be effectively captured by a reward model.
We propose \method, a framework that trains domain-invariant reward models by optimizing a dual loss: a domain loss that minimizes the divergence between source and target distribution, and a source loss that optimizes preferences on the source domain.
We show \method is a general approach that we evaluate and analyze across 4 distinct settings: (1) Cross-lingual transfer (accuracy: $0.621 \rightarrow 0.661$), (2) Clean-to-noisy (accuracy: $0.671 \rightarrow 0.703$), (3) Few-shot-to-full transfer (accuracy: $0.845 \rightarrow 0.920$), and (4) Simple-to-complex tasks transfer (correlation: $0.508 \rightarrow 0.556$).
Our code, models and data are available at \url{https://github.com/portal-cornell/dial}.

\end{abstract}

\section{Introduction}
Reinforcement Learning from Human Feedback (RLHF) has emerged as a popular paradigm for aligning language models~\citep{Ouyang2022TrainingLM, dubey2024llama}. This approach involves training and optimizing a reward model that learns human preferences. However, the effectiveness of RLHF is limited by the ability to collect high-quality feedback data. As tasks become more complex, they require greater human expertise and time, making it harder for humans to supervise and provide feedback~\citep{leike2018scalable}.

We address the problem of learning reward models for target domains that lack human preference feedback. While feedback is unavailable in the target domain, it is often easy to collect on related source domains. For example, extensive preference data is available in English (source domain) across various tasks, whereas low-resource languages (target domain) may have little to no labeled data~\citep{costa2022nllb}. Similarly, preferences are easier to collect for simpler tasks, like rating article summaries~\cite{KaggleArgument}, compared to more complex tasks, such as evaluating full-length articles~\citep{KaggleEssay}.

\begin{figure*}[t]
\centering
\includegraphics[width=0.9\linewidth]{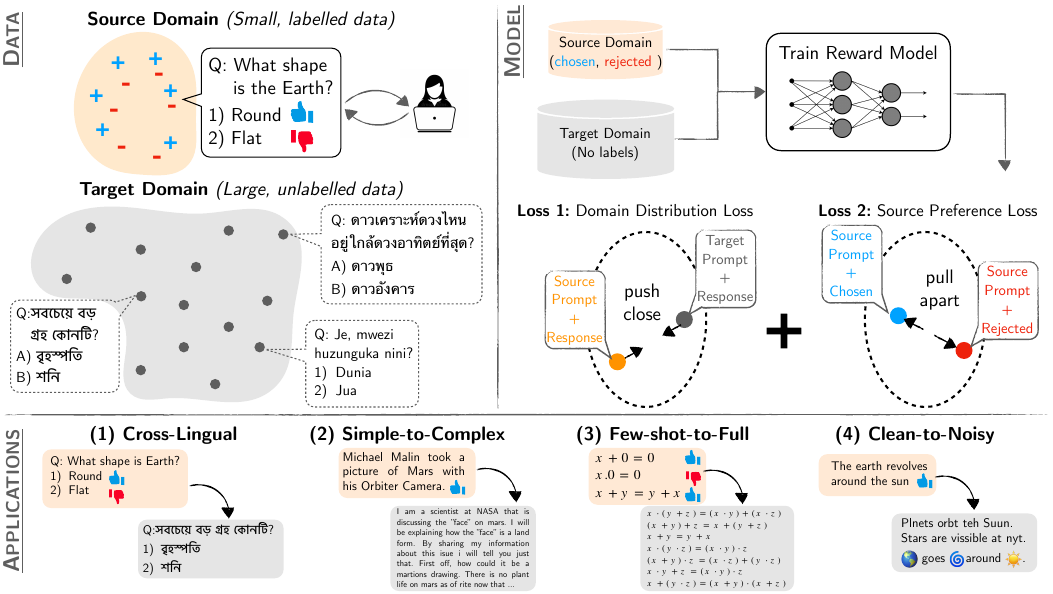}
\caption{\method trains domain-invariant reward model for target domains with no labeled preference data. \method leverages labeled source data and unlabeled target data to train reward models on a dual loss: a domain loss that minimizes the divergence between source and target distribution, and a source loss that optimizes preferences on the source domain. 
We show \method is a general approach that we evaluate and analyze across 4 distinct applications: (1) Cross-lingual transfer, (2) Clean-to-noisy, (3) Few-shot-to-full transfer, and (4) Simple-to-complex tasks transfer.\figGap }
\label{fig:tasks}
\end{figure*}

Prior works address the problem of no target domain data through methods such as regularizing with a text-generation loss~\citep{Yang2024RegularizingHS,zhang2024generative}, pre-training on unlabeled target data~\citep{karouzos2021udalm}, or few-shot prompting~\citep{winata2022cross}. However, both regularization and pre-training are surrogate objectives that don't guarantee the reward model learns the correct preferences on the target domain. Finally, few-shot learning is often sensitive to the choice of examples, leading to high variance in performance.

Our key insight is that, \textbf{while domains may differ significantly, human preferences convey \emph{domain-agnostic} concepts that can be effectively captured by a reward model.}
By designing the reward model to disentangle domain-specific features from concepts, we enable transfer across domains.
To achieve this, we train on a dual loss: a domain loss that minimizes divergence between source and target distributions, and a source loss that learns preferences on the source domain.

We propose \textbf{D}omain \textbf{I}nvariant \textbf{A}lignment for \textbf{L}anguage (\method), a framework for training domain-invariant reward models. \method takes as input labeled source data and unlabeled target data and trains a language model with two heads -- a critic head and a reward head. The critic head is trained adversarially on a domain loss to minimize the Wasserstein distance (WD)~\cite{Arjovsky2017WassersteinG} between source and target embeddings while the reward head minimizes a source loss that optimizes preferences on the source data. Effectively, the domain loss aligns source and target embeddings, while the source loss separates chosen and reject embeddings, encouraging the reward model to learn preferences in a domain-invariant manner.

We demonstrate that \method is a general approach that can be used to do domain transfer in multiple different paradigms. Specifically, we evaluate and analyze the following settings: \textit{(a) Cross-lingual Transfer:} transferring preferences from a high-resource language (e.g., English) to a low-resource language (e.g., Korean); \textit{(b) Clean-to-Noisy Transfer:} adapting preferences from clean, structured data to noisy, real-world data (e.g., internet text with slang or emojis), (\textit{c) Few-shot-to-full Transfer:} leveraging limited labeled examples to generalize across a broader, unlabeled target distribution, and (\textit{d) Simple-to-complex Transfer:} aligning preferences from simpler tasks (e.g., short texts) to more challenging tasks (e.g., long-form content). Our key contributions are:
\begin{enumerate}[nosep, leftmargin=0.1in]
\item A novel framework, \method, for training domain-invariant reward models. \method transfers preferences from labeled source to unlabeled target domains.
\item Theoretical and empirical performance analysis of
\method reward models on target domain distributions, including scaling 
\item Evaluation across four distinct applications:
    \begin{enumerate}[nosep, leftmargin=0.1in]
        \item Cross-lingual transfer from English to $3$ languages on Stanford Human Preference~\cite{pmlr-v162-ethayarajh22a} (accuracy: $0.621 \rightarrow 0.661$).
        \item Clean-to-noisy transfer from grammatically correct to noisy internet posts on corrupted Stanford Human Preference~\cite{pmlr-v162-ethayarajh22a} dataset (accuracy: $0.671 \rightarrow 0.703$).
        \item Few-shot-to-full transfer from 10 examples on CValues~\cite{xu2023cvalues} safety dataset (accuracy: $0.845 \rightarrow 0.920$).
        \item Simple-to-complex transfer from scoring short argument fragments to long student essays on Kaggle~\cite{KaggleEssay} (correlation: $0.508 \rightarrow 0.556$).
    \end{enumerate}
\end{enumerate}

\begin{figure*}[!t]
\centering
\includegraphics[width=\linewidth]{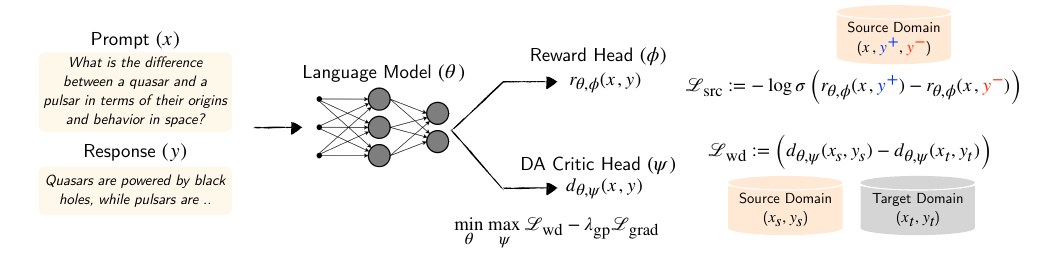}
\caption{\textbf{\method overview.} \method takes labeled source and unlabeled target data and trains a domain-invariant reward model. The model takes prompt ($x$) and response ($y$), passes it through a base language model ($\theta$) with two heads: a domain critic head ($\psi$) and a reward head ($\phi$). The critic head is trained adversarially to minimize the Wasserstein distance between source and target embeddings while the reward head optimizes preferences on source data. \figGap}
\label{fig:da_method}
\end{figure*}

\section{Approach}
We present Domain Invariant Alignment for Language (\method), a framework for aligning large language models across domains where human preference feedback data is unavailable. Given a small labeled preference dataset on a source domain $\mathcal{D_{\rm src}} = \{ (x, y^+, y^-)\}$ and a large unlabeled dataset on a target domain $\mathcal{D_{\rm tgt}} = \{x, y\}$, \method trains a domain-invariant reward model $r_\theta(x,y)$ to achieve strong performance on both source and target domains. To achieve this, we train on a dual loss: a domain loss that minimizes the Wasserstein Distance between source and target distribution, and a source loss that optimizes preferences on the source domain. Sec.~\ref{sec:approach:reward_model} defines the domain-invariant reward model and Sec.~\ref{sec:approach:alg} details \method. 

\subsection{Domain-invariant Reward Model}
\label{sec:approach:reward_model}
We introduce a domain-invariant reward model that enables transferring preferences from a labeled source domain to an unlabeled target domain. The model takes a base LLM, removes the final unembedding layer, and adds two scalar output heads: a domain critic head and a reward head. Fig.~\ref{fig:da_method} provides an overview of our architecture.

\paragraph{Domain Critic Head.}
The domain critic head, denoted as $d_{\theta, \psi}(x, y)$, maps a prompt $x$ and response $y$ to a scalar score, which is used to measure the distributional distance between source and target distributions. It takes the embedding of $(x,y)$ from the language model $\theta$, and passes it through an MLP head $\psi$ to compute a scalar score.

We choose the Wasserstein distance~\citep{Arjovsky2017WassersteinG}, a metric that measures the minimum cost of transporting one probability distribution to match another. Unlike KL divergence, which requires overlapping supports and can be undefined when distributions do not overlap, the Wasserstein distance provides a meaningful comparison even when the distributions have disjoint supports, making it well-suited for domain adaptation. The 1-Wasserstein distance can be expressed in its dual formulation~\citep{villani2009optimal} as:
\begin{equation} 
W_1(\mathbb{P}, \mathbb{Q}) = \sup_{ ||f||_L \leq 1} \mathbb{E}_{z \sim \mathbb{P}}[f(z)] - \mathbb{E}_{z \sim \mathbb{Q}}[f(z)]
\end{equation} 
 where $|f||_L \leq 1$ is the set of 1-Lipschitz functions.

We train the critic $\psi$ to approximate the Wasserstein distance between source and target distributions by maximizing the expected score difference between source and target:
\begin{align*}
\max_\psi \mathcal{L}_{\rm wd}(\theta, \psi) = & \mathbb{E}_{(x_s, y_s) \sim \mathcal{D}_{\rm src}} \left[ d_{\theta, \psi}(x_s, y_s) \right] - \\ 
& \mathbb{E}_{(x_t, y_t) \sim \mathcal{D}_{\rm tgt}} \left[ d_{\theta, \psi}(x_t, y_t) \right] \numberthis
\end{align*}
where $d_{\theta, \psi}(x, y)$ models the feature functions $f$. To ensure that $d_{\theta, \psi}(x, y)$ satisfies the Lipschitz constraint, we impose a gradient penalty~\citep{Gulrajani2017ImprovedTO} on the critic: 
\begin{equation}
\mathcal{L}_{\rm grad}(\psi) = \mathbb{E}_{(x,y)} \left[\big(\|\nabla_{x,y} d_{\theta, \psi}(x, y)\| - 1\big)^2\right] 
\end{equation}
where critic gradients are penalized not only at source and target embeddings, but at random interpolates between the two. The critic maximizes a weighted difference of $\mathcal{L}_{\rm wd} - \lambda_{gp} \mathcal{L}_{\rm grad}$.

We then update the language model embeddings $\theta$ to minimize the Wasserstein distance, i.e., $\min_\theta \mathcal{L}_{\rm wd}(\theta, \psi)$ while keeping the critic frozen. 
This results in the following adversarial game between the critic and the language model:
\begin{equation}
\min_\theta \max_\psi \mathcal{L}_{\rm wd} (\theta, \psi) - \lambda_{gp} \mathcal{L}_{\rm grad} (\psi)
\end{equation}

The equilibrium of the game is reached if the language model $\theta$ finds embeddings where the source and target data are indistinguishable, thus being domain-invariant.

\paragraph{Reward Head.}
The reward head, denoted as $r_{\theta, \phi}(x, y)$, maps a prompt $x$ and response $y$ to a scalar reward. It takes the embedding of $(x,y)$ from the language model $\theta$, and passes it through a linear head $\phi$ to compute a scalar reward~\cite{Ouyang2022TrainingLM}. Given a labeled source preference dataset $\mathcal{D}_{\rm source} = (x, y^+, y^-)$, we train both the reward head and the embedding using a Bradley-Terry model~\citep{Bradley1952RankAO} on the following source loss:
\begin{equation}
    \begin{aligned}
         \mathcal{L}_\text{src}(\theta, \phi) = \mathbb{E}_{(x, y^+, y^-) \sim \mathcal{D}_{\rm src}} 
    \big[ \log \sigma \big(& r_{\theta,\phi}(x, y^+) \\
    -& r_{\theta,\phi}(x, y^-) \big) \big]
    \end{aligned}
\end{equation}
where $\sigma$ is the sigmoid function. This loss encourages the reward head to maximize the difference in rewards between preferred and rejected responses.

\subsection{\method algorithm}
\label{sec:approach:alg}
\begin{algorithm}[!htbp]
\caption{\method: \small Learning Domain Invariant Rewards}
\label{alg:training}
\begin{lstlisting}[style=compactpython]
# Inputs:
#  mixed_dataloader: Yields source and target data
#  lm: Language model producing embeddings
#  critic: MLP critic computing Wasserstein Distance
#  reward: Linear reward head 
for mixed_batch in mixed_dataloader:
  # Load labeled source, unlabeled target
  src_chosen, src_reject, tgt_all = mixed_batch
  src_all = cat([src_chosen, src_reject])

  # Critic maximizes source-target dist 
  lm.requires_grad_(False)
  src_emb, tgt_emb = lm(src_all), lm(tgt_all)
  wd = (critic(src_emb) - critic(tgt_emb)).mean()
  gp_loss = grad_penalty(critic, src_emb, tgt_emb)
  critic_loss = -wd + gp_loss

  # Embeddings minimize source-target dist 
  lm.requires_grad_(True)
  critic.requires_grad_(False)
  da_loss = wd

  # Rewards minimize source preference loss
  ch_emb, rj_emb = lm(src_chosen), lm(src_reject)
  ch_rew, rj_rew = reward(ch_emb), reward(rj_emb)
  src_loss = -F.logsigmoid(ch_rew - rj_rew).mean()

  total_loss = src_loss + da_loss + critic_loss
\end{lstlisting}
\end{algorithm}

Algorithm~\ref{alg:training} describes the \method algorithm. At every iteration, the critic head is updated to maximize the Wasserstein distance between source and target embeddings while enforcing a Lipschitz constraint using a gradient penalty. Next, the language model is updated to minimize the Wasserstein distance, aligning the source and target embeddings. Finally, the reward head minimizes a preference loss on the source data to separate chosen and rejected responses. By alternating between these updates, \method learns a reward model that transfers preferences from the source to the target domain.

\subsection{Theoretical Analysis}
We now analyze the generalization properties of the \method reward model $r(x, y)$. Let $f(x,y)$ denote the ground-truth reward function which assigns scores to prompt-response pairs. To measure the alignment between $r$ and $f$, we consider pairwise preferences derived from triplets $(x, y, y')$, where $y, y'$ are two responses to the prompt $x$. The preference induced by $f$ is $f(y_{\rm{win}}) \geq f(y_{\rm{loss}})$.
The error of reward model $r$ in a domain $\mathcal{D}$ is the expected disagreement between $r$ and $f$ on a distribution $\mathcal{D}$ defined using the Bradley-Terry loss:
\begin{equation}
    \epsilon_\mathcal{D}(r, f) = \mathbb{E}_{(x, y, y') \sim \mathcal{D}} \big[ \sigma \big( r(x, y_{\text{loss}}) - r(x, y_{\text{win}}) \big) \big]
\end{equation}
where $\sigma(z) = \frac{1}{1 + e^{-z}}$ is the sigmoid function. This error measures the probability that $r$ disagrees with $f$, with $\epsilon_\mathcal{D}(r, f) \to 0$ as $r$ aligns perfectly with $f$.

We show that performance of \method reward on the target domain is bounded by sum of performance on the source domain and the Wasserstein distance between source and target:
\begin{theorem}
\label{theo:generalization_bound}
Let $r$ be a $K$-Lipschitz function. Then the target 
domain error $\epsilon_T(r, f)$ satisfies:
\begin{equation}
    \epsilon_T(r, f) \leq \epsilon_S(r, f) + 2K L_\sigma W_1(\mu_S, \mu_T),
\end{equation}
where $W_1(\mu_S, \mu_T)$ is the Wasserstein-1 distance between the source and target distributions $\mu_S$ and $\mu_T$ over $(x, y)$, and $L_\sigma=\frac{1}{4}$ is the Lipschitz constant of $\sigma$.
\end{theorem}

See Appendix~\ref{appendix:theory} for the detailed proof. The two terms on the right hand side correspond to the source and domain loss in \method. By minimizing the sum, \method bounds the target performance. 

\begin{table*}[!t]
\centering
\resizebox{\textwidth}{!}{%
\begin{tabular}{lccccccccc}
\toprule
 & \multicolumn{3}{c}{\texttt{legaladvice}} & \multicolumn{3}{c}{\texttt{askscience}} & \multicolumn{3}{c}{\texttt{explainlikeimfive}} \\
\cmidrule(lr){2-4} \cmidrule(lr){5-7} \cmidrule(lr){8-10}
\textbf{Method} & \textbf{Korean} & \textbf{Thai} & \textbf{Chinese} & \textbf{Korean} & \textbf{Thai} & \textbf{Chinese} & \textbf{Korean} & \textbf{Thai} & \textbf{Chinese} \\
\midrule
\texttt{Base LM} & $0.58$ & $0.57$ & $0.60$ & $0.57$ & $0.58$ & $0.55$ & $0.55$ & $0.54$ & $0.51$ \\
\texttt{Src-Pref} & $0.60$ & $0.63$ & $0.61$ & $0.57$ & $0.62$ & $0.59$ & $0.65$ & $0.63$ & \cellgreen $\mathbf{0.68}$ \\
\texttt{Src-Pref-SFT}~\textcolor{perfred}{[1]} & $0.62$ & $0.59$ & $0.64$ & $0.56$ & $0.61$ & $0.56$ & $0.65$ & $0.64$ & $0.61$ \\
\texttt{Src-Pref-Tgt-NTP}~\textcolor{perfred}{[2]} & $0.64$ & $0.56$ & $0.65$ & $0.57$ & $0.63$ & $0.61$ & $0.64$ & $0.64$ & $0.61$ \\ 
\midrule
\method (ours) & \cellgreen $\mathbf{0.68}$ & \cellgreen $\mathbf{0.66}$ & \cellgreen $\mathbf{0.68}$ & \cellgreen $\mathbf{0.63}$ & \cellgreen $\mathbf{0.68}$ & \cellgreen $\mathbf{0.62}$ & \cellgreen $\mathbf{0.68}$ & \cellgreen $\mathbf{0.65}$ & $0.67$ \\
\midrule
\rowcolor{hgray} \texttt{Tgt-Pref*} & $0.69$ & $0.66$ & $0.67$ & $0.62$ & $0.67$ & $0.62$ & $0.67$ & $0.65$ & $0.68$ \\
\rowcolor{hgray} \texttt{Src-Tgt-Pref*} & $0.69$ &  $0.72$ &  $0.70$ & $0.63$ & $0.67$ & $0.64$ &  $0.70$ & $0.67$ &  $0.69$ \\
\bottomrule
\end{tabular}
}
\caption{\textbf{Cross-lingual Transfer.} Accuracy results of reward models trained on source data (English) and evaluated on target data (Korean/Thai/Chinese) on three splits of Stanford Human Preference Dataset~\citep{pmlr-v162-ethayarajh22a}, each 1K. Results are averaged over 3 seeds. \method outperforms all baselines, including \textcolor{perfred}{[1]}~\cite{Yang2024RegularizingHS} and \textcolor{perfred}{[2]}~\cite{karouzos2021udalm} \figGap}
\label{tab:cross_lingual}
\end{table*}

\section{Experiments}
\subsection{Experiment Setup}
\paragraph{Baselines.} 
We compare against various baselines. \texttt{Src-Pref} is a reward model trained on preference data on the source domain.  \texttt{Src-Pref-SFT}~\cite{Yang2024RegularizingHS} trains a reward model on source preference data, and additionally regularizes the base model with SFT loss on chosen responses on the source data. 
\texttt{Src-Pref-Tgt-NTP}~\cite{karouzos2021udalm} trains a reward model on source preference data, and additionally regularizes the base model with a pre-training task on both prompt and response on target data.\footnote{The original paper uses a masked language modeling task. In order to make it comparable for our decoder-only models, we used next token prediction instead of masked language modeling.} 
\texttt{Base LM} is a generative baseline that prompts the base model to choose from multiple responses, leveraging chain-of-thought.
We also include two oracle baselines: \texttt{Tgt-Pref*} trains a reward model on target preference data, \texttt{Src-Tgt-Pref*} trains a reward model on both source and target preference data.

\paragraph{Model.} 
We use a base model Gemma-2b~\citep{Mesnard2024GemmaOM} with a learned linear reward head and a 2-layer MLP critic head. We fine-tuned with LoRA~\citep{Hu2021LoRALA}. See Appendix~\ref{appendix:general} for a more detailed model architecture.

\paragraph{Metrics.} We measure accuracy of the reward model on preference datasets that contain prompt, chosen, and rejected responses. See Appendix~\ref{appendix:details} for detailed table with variances.

\subsection{Application 1: Cross-lingual Transfer}

\paragraph{Setup.}
We first look at cross-lingual transfer where preferences exist in a high-resource source domain but must be transferred to a low-resource target domain with costly labels.
Stanford Human Preferences~\citep{pmlr-v162-ethayarajh22a} (SHP) is a dataset consisting of questions from Reddit and pairs of preferred and non-preferred answers. We select 3 diverse subreddits with train/test splits: \texttt{legaladvice} (20K/1K), \texttt{explainlikeimfive} (20K/1K) and \texttt{askscience} (13K/1K). To evaluate cross-lingual transfer, we translate data points using NLLB~\citep{team2022NoLL} to 3 low-resource languages: Korean, Thai, Chinese. See Appendix~\ref{sec:appendix:cross}.

\paragraph{Results.}
Table.~\ref{tab:cross_lingual} shows that \method outperforms baselines on all subreddits and languages. 
The performance improvements on Korean ($0.57\rightarrow0.63$) and Thai ($0.62\rightarrow0.68$) are stronger than Chinese ($0.59\rightarrow0.62$), likely due to Chinese being more common in training datasets of the base LM. Performance improvements on subreddits \texttt{legaladvice} ($0.63\rightarrow0.66$) and \texttt{askscience} ($0.62\rightarrow0.68$) are stronger than \texttt{explainlikeimfive} ($0.64\rightarrow0.65$), likely because legal advice and ask science may rely on more advanced terminology that requires more alignment.
We note the single case of \texttt{explainlikeimfive}-Chinese where \method does not improve over baselines. This is because training on source already matches oracle performance, leaving little to no room for improvement in transfer ability using domain adaptation. 
Finally, we note that on many subreddits/language \method reaches oracular performance of \texttt{Src-Tgt-Pref*}.

\subsection{Application 2: Clean-to-noisy Transfer}

\begin{table}[t]
\centering
\resizebox{\columnwidth}{!}{%
\begin{tabular}{lccc}
\toprule
\textbf{Method} & \textbf{Legal} & \textbf{Science} & \textbf{ELI5} \\
\midrule
\texttt{Base LM}  & $0.55$ & $0.56$ & $0.55$ \\
\texttt{Src-Pref} & $0.71$ & $0.63$ & $0.67$ \\
\texttt{Src-Pref-SFT}~\textcolor{perfred}{[1]} & $0.71$ & $0.61$ & $0.64$ \\
\texttt{Src-Pref-Tgt-NTP}~\textcolor{perfred}{[2]} & $0.70$ & $0.61$ & $0.63$ \\ \midrule
\method (ours) & \cellgreen  \allbold{$0.76$} & \cellgreen 
 \allbold{$0.65$} & \cellgreen  \allbold{$0.70$} \\
\midrule
\rowcolor{hgray} \texttt{Tgt-Pref*} & $0.77$ & $0.67$ & $0.74$ \\
\rowcolor{hgray} \texttt{Src-Tgt-Pref*} & $0.78$ & $0.69$ & $0.73$ \\
\bottomrule
\end{tabular}}
\caption{\textbf{Clean-to-noisy Transfer.} Accuracy results of reward models trained on clean data and evaluated on noisy data from three splits of SHP dataset~\citep{pmlr-v162-ethayarajh22a}, each 1K. Results averaged over 3 seeds. \method outperforms all baselines, including \textcolor{perfred}{[1]}~\cite{Yang2024RegularizingHS} and \textcolor{perfred}{[2]}~\cite{karouzos2021udalm} \figGap}
\label{tab:noise}
\end{table}

\paragraph{Setup.} 
We next look at the application where the source domain is clean, high-quality synthetic data, but the target domain is noisy, real-world data, such as the internet. To emulate this, we selected the same three splits from Stanford Human Preferences~\citep{pmlr-v162-ethayarajh22a} (SHP): \texttt{legaladvice}, \texttt{explainlikeimfive}, and \texttt{askscience}. We used the original test data (1K each) which is already noisy. We create a ``clean'' dataset (20K/20K/13K) by rewriting the data to be formal using Gemma-2-9b-it~\citep{Riviere2024Gemma2I}.  The real-world data contains significant noise in the form of spelling, grammar, and language errors. See Appendix~\ref{sec:appendix:noise}.

\paragraph{Results.} 
Table~\ref{tab:noise} shows that \method outperforms baselines on all 3 splits, matching oracle accuracy on \texttt{legaladvice}. The gain on \texttt{ELI5} is the least as the rewritten data is also fairly informal given the nature of ELI5. Finally, we note that baseline \texttt{Src-Pref-Tgt-NTP} performs similar or slightly worse than \texttt{Src-Pref}, likely because the base LLM has been pretrained on noisy internet data, but is insufficient for reward generalization.

\subsection{Application 3: Few-shot-to-full Transfer}
\paragraph{Setup.} We next look at the application where the source data is a set of few-shot labeled examples and the target data is unlabeled data from the same distribution. This is often the case when labeling data might be expensive as it requires domain experts. We select the CValues safety dataset~\citep{xu2023cvalues} and sample three splits of $10$ examples as source data. We train on 130K unlabelled target data and evaluate on 7.5K labeled data. Our goal is to test \method under low labeled data regimes and quantify performance gains over \texttt{Src-pref} when source and target domain are already aligned. See Appendix~\ref{sec:appendix:few}.

\begin{table}[t]
\centering
\resizebox{\columnwidth}{!}{%
\begin{tabular}{lccc}
\toprule
\textbf{Method} & \textbf{SplitA} & \textbf{SplitB} & \textbf{SplitC} \\
\midrule
\texttt{Few-shot LM} & $0.60$ & $0.64$ & $0.65$  \\ 
\texttt{Src-Pref} & $0.82$ & $0.86$ & $0.85$  \\
\texttt{Src-Pref-SFT}~\textcolor{perfred}{[1]} & $0.74$ & $0.79$ & $0.71$  \\
\texttt{Src-Pref-Tgt-NTP}~\textcolor{perfred}{[2]} & $0.76$ & $0.72$ & $0.72$  \\ \midrule
 \method (ours) & \cellgreen $\mathbf{0.85}$ & \cellgreen $\mathbf{0.97}$ & \cellgreen $\mathbf{0.94}$  \\
\midrule
\rowcolor{hgray} \texttt{Tgt-Pref*} & 1.00 & 1.00 & 1.00  \\
\bottomrule
\end{tabular}
}
\caption{\textbf{Few-shot-to-full Transfer.} Accuracy results on three splits of CValues safety preference dataset~\citep{xu2023cvalues} of reward models trained on few-shot source examples, evaluated on 7.5K target examples. \method outperforms all baselines, including \textcolor{perfred}{[1]}~\cite{Yang2024RegularizingHS} and \textcolor{perfred}{[2]}~\cite{karouzos2021udalm}. \figGap}
\label{tab:fewshot}
\end{table}

\paragraph{Results.}
Table~\ref{tab:fewshot} shows that \method outperforms all baselines on all splits, coming close to oracle acccuracy ($1.0$) on 2 out of 3 splits. \method accuracy does vary across splits due to sensitivity to the choice of a small number of few-shot examples, but it still strictly outperforms the baselines. \texttt{Few-shot LLM}, which uses the examples in-context, performs poorly likely due to a smaller base LLM being unable to effectively use in-context learning.

\subsection{Application\,4:\,Simple-to-complex Transfer}

\begin{table}[t]
\centering
\resizebox{\columnwidth}{!}{%
\begin{tabular}{lcc}
\toprule
\textbf{Method} & \textbf{Pearson's $r$} & \textbf{Spearman's $\rho$} \\
\midrule
\texttt{Base LM} & $0.408$ \textcolor{gray}{\scriptsize{$\pm 0.003$}} & $0.394$ \textcolor{gray}{\scriptsize{$\pm 0.003$}} \\
\texttt{Src-Pref}  & $0.516$ \textcolor{gray}{\scriptsize{$\pm 0.011$}} & $0.508$ \textcolor{gray}{\scriptsize{$\pm 0.014$}} \\
\texttt{Src-Pref-SFT}~\textcolor{perfred}{[1]} & $0.567$ \textcolor{gray}{\scriptsize{$\pm 0.018$}} & \cellgreen  $\mathbf{0.562}$ \textcolor{gray}{\scriptsize{$\pm 0.023$}} \\
\texttt{Src-Pref-Tgt-NTP}~\textcolor{perfred}{[2]} & $0.567$ \textcolor{gray}{\scriptsize{$\pm 0.020$}} & $0.558$ \textcolor{gray}{\scriptsize{$\pm 0.025$}} \\
 \method (ours) & \cellgreen $\mathbf{0.577}$ \textcolor{gray}{\scriptsize{$\pm 0.011$}} & $0.556$ \textcolor{gray}{\scriptsize{$\pm 0.011$}} \\
\midrule
\rowcolor{hgray} \texttt{Tgt-Pref*} & $0.857$ \textcolor{gray}{\scriptsize{$\pm 0.003$}} & $0.855$ \textcolor{gray}{\scriptsize{$\pm 0.003$}} \\
\rowcolor{hgray} \texttt{Src-Tgt-Pref*} & $0.860$ \textcolor{gray}{\scriptsize{$\pm 0.001$}} & $0.857$ \textcolor{gray}{\scriptsize{$\pm 0.001$}} \\
\bottomrule
\end{tabular}}
\caption{\textbf{Simple-to-complex Transfer.} Correlation results of ratings of reward models trained on short KAggle argument~\citep{KaggleArgument} and evaluated on long Kaggle essay~\citep{KaggleEssay} on 1K datapoints. Results show mean and standard deviation over 3 seeds. 
\method outperforms all baselines, including \textcolor{perfred}{[1]}~\cite{Yang2024RegularizingHS} and \textcolor{perfred}{[2]}~\cite{karouzos2021udalm}, on Pearsons'$r$ but underperforms on Spearman's $\rho$. \figGap}
\label{tab:shorttolong}
\end{table}

\begin{figure*}[!t]
\centering
\includegraphics[width=0.9\textwidth]{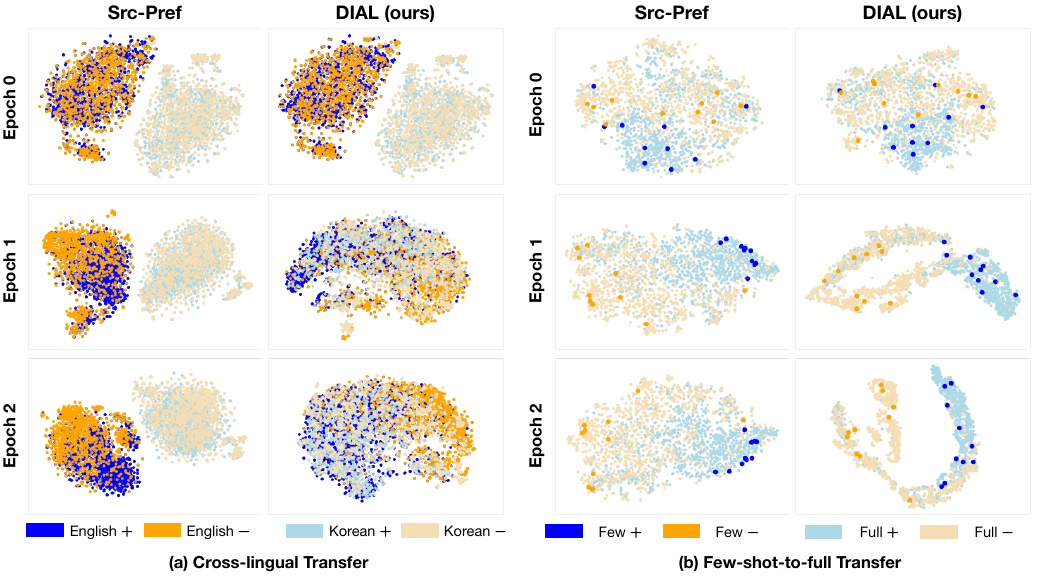}
\caption{\textbf{Reward model embeddings} learned by \method and \texttt{Src-Pref} across training iterations on (a) Cross-lingual Transfer and (b) Few-shot-to-full Transfer. \texttt{Src-Pref} separates source embeddings, but not target embeddings, resulting in poor transfer. \method learns embeddings that cluster (source positive, target positive) and (source negative, target negative) allowing for better reward transfer. \figGap} 
\label{fig:embedding}
\end{figure*}

\paragraph{Setup.}
We finally look at the application where the source data is simpler and easier to label, where target data is complex and expensive to label. Specifically, we use Kaggle Argument~\citep{KaggleArgument} as source data (20K train) and Kaggle Essay (long)~\citep{KaggleEssay} as target data (15K train / 1K val). The source dataset consists of individual fragments of student essays corresponding to different components of an argument (e.g. Thesis, Evidence, Rebuttal) and ranges from approximately $10$ to $100$ words per example. The target dataset consists of full student essays on a variety of topics (e.g. Mars, Electoral College), and is around $200$ to $600$ words per essay. Each datapoint is a prompt, a response, and a score from a human annotator. We transform the source data to a preference dataset to train a reward model, and at test time ask the reward model to score responses. We compute two correlation metrics -- Pearson's $r$ and Spearman's $\rho$. Pearson measures linear correlation between the reward model and human scores, while Spearman's measures if the order is retained.  The goal is to test \method on transfer across data of significantly different lengths and complexity. See Appendix~\ref{sec:appendix:short}.

\paragraph{Results.}
Table~\ref{tab:shorttolong} shows that \method improves upon \texttt{Src-Pref} on all metrics. While \method improves upon all baselines on Pearson's $r$, \texttt{Src-Pref-SFT} achieves the highest performance on Spearman's $\rho$. We also note that the oracle methods have a much stronger performance compared to other applications. It's likely due to the complex tasks being sufficiently different from the source, such that aligning embeddings between the two is challenging. Nonetheless, the target data does help \method achieve better performance with lower variance. While these results show promise that \method can aid in providing scalable oversight, there is sufficient headroom for future work to further improve over \method to come close to oracle performance.

\subsection{What reward model embeddings does \method learn?}
To understand why \method rewards generalize, we visualize the reward model embeddings for different applications. We compute a t-SNE mapping of the embeddings and apply it to $1000$ random target datapoints. Fig.~\ref{fig:embedding} shows embeddings for both \method and baseline \texttt{Src-Pref} across training epochs. 

In cross-lingual transfer, we analyze the \texttt{legaladvice}-Korean split in Fig.~\ref{fig:embedding}(a), where source is English and target is Korean. At epoch 0, source and target data are well separated, but positive and negative responses are not. In \texttt{Src-Pref}, the source preference loss continues separating source positives and negatives over time. However, target positives and negatives still remain mixed. \method, on the other hand, aligns source and target embeddings, and uses this alignment to simultaneously separate (source positive, target positive) from (source negative, target negative). This alignment enables it to easily transfer preferences.

We see a similar behavior in few-shot-to-full transfer shown in Fig.~\ref{fig:embedding}(b), where the source is a few shot examples and the target is the full dataset. Both methods easily separate the source data. However, \texttt{Src-Pref} overfits to the source examples and finds an incorrect decision boundary that doesn't separate the full dataset. On the other hand, \method aligns the few-shot examples with the full training dataset (clustering the training examples around the few-shot examples), leading to a clear decision boundary that generalizes to the full dataset. See Appendix~\ref{sec:appendix:boundary} for further analysis.
\vspace{-0.65em}
\subsection{How does \method scale with data?}

We analyze how \method scales with data on the cross-lingual task of \texttt{legaladvice}-Korean. We hypothesize that for a fixed data budget, there is an optimal mix of labeled source and unlabeled target. From theory, we can bound performance on target by performance on source (which depends on $1/\sqrt{N_{\rm src}}$) and estimated Wasserstein distance between source and target (which depends on $1/\sqrt{N_{\rm src}} + 1/\sqrt{N_{\rm tgt}}$).
Fig.~\ref{fig:dial_data_scaling} (a) shows that there is indeed such a peak around $(0.4, 0.6)$ mix of (target, source). 

\begin{figure}[!htbp]
\centering
\includegraphics[width=\columnwidth]{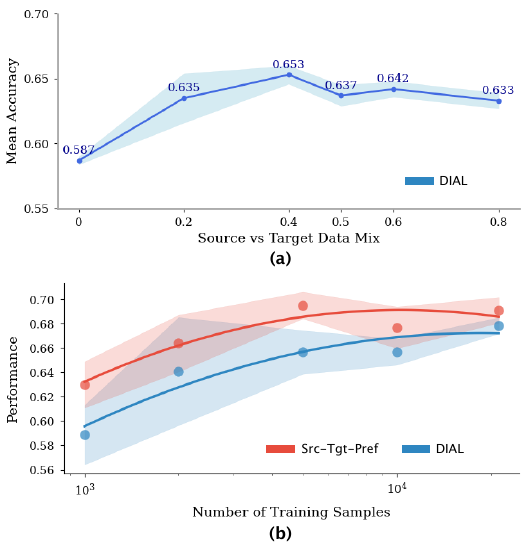}
\caption{\textbf{Scaling with Data} on \texttt{legaladvice}-Korean split. (a) \method performance with varying source-target data mix (b) \method scaling with unlabeled target data vs \texttt{Src-Tgt-Pref} scaling with labeled target data. Resuls on 3 seeds. \figGap}
\label{fig:dial_data_scaling}
\end{figure}

We also study scaling laws for \method with unlabeled target data vs oracle (\texttt{Src-Tgt-Pref*}) with labeled target data. As expected the oracle clearly has an offset from \method, but with more data, \method catches up. This is likely because the oracle asymptotes, while \method uses the additional data to perfectly align with source and transfer rewards.

\subsection{How robust is \method to spurious rewards?}
We hypothesize that reward models trained only on source preference data via \texttt{Src-Pref} are susceptible to learning spurious rewards. This is a well-studied problem in reward learning~\cite{tien2022causal}, where causal confounders or biases in training data can lead to ``reward confusion''. We also hypothesize that \method, with only unlabeled target data, can learn rewards that avoid such correlation.  

\begin{figure}[!htbp]
\centering
\includegraphics[width=\columnwidth]{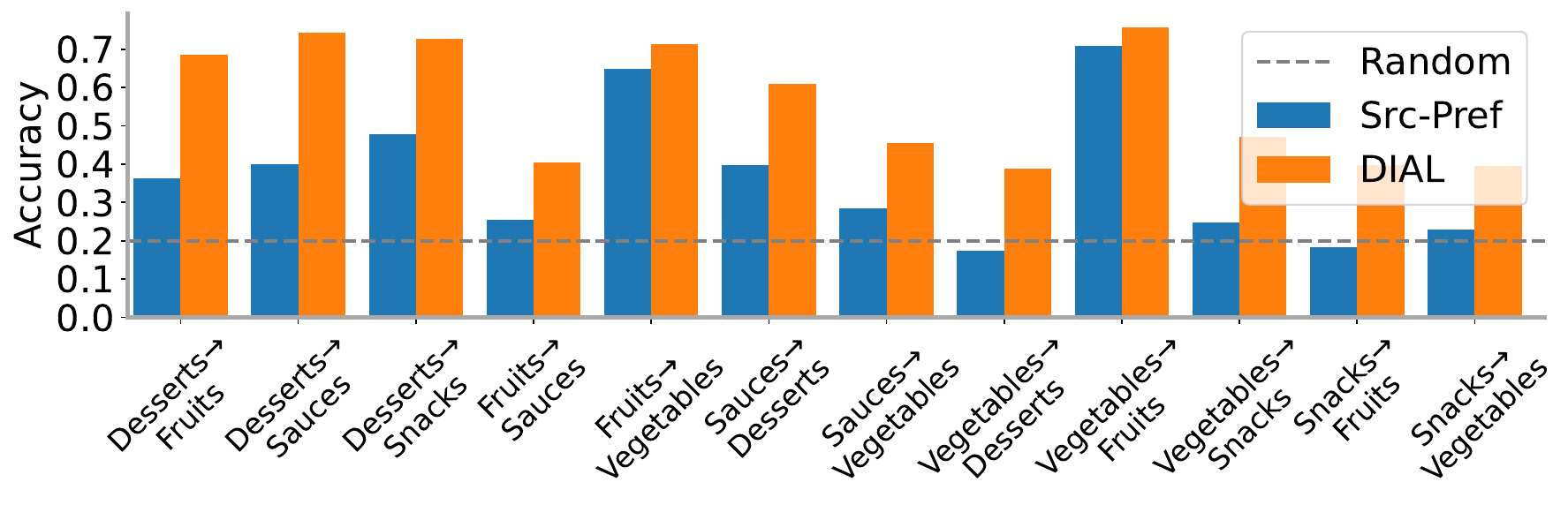}
\caption{\textbf{Spurious reward.} Accuracy results on odd-one-out ($100$ datapoints) over 3 seeds. \method learns the correct reward, while \texttt{Src-Pref} learns spurious reward of ``not source'' which performs similar to random. 
\figGap}
\label{fig:ooo}
\end{figure}

To introduce spurious correlations, we create a synthetic task of choosing the odd one out. We generated sets of 100 items belonging to 5 groups of foods: desserts, fruits, sauces, vegetables, and snacks. We create datasets for each category, where each data point consists of 4 items in the category and 1 from another. The spurious reward here is ``not the source category'' (e.g. not fruit, not vegetable), which excels on source but fail on target.

Fig.~\ref{fig:ooo} shows \method vs \texttt{Src-Pref} across multiple tasks. We see that \texttt{Src-Pref} performs poorly on target, often the same as random ($0.2$), likely due to learning the spurious reward of ``not source category''. \method on the other hand generalizes to target domain, despite seeing no target labels. See Appendix~\ref{sec:appendix:ooo} for more details.

\subsection{Can \method adjust for distributions shift during RLHF training?}

We next look at how \method can help in training better policies during RLHF. A common problem in RLHF is distribution shift during the training process, where the reward model trained on off-policy preferences over time becomes inaccurate under the current policies response distribution~\cite{ziegler2019fine}. This typically requires repeatedly collecting on-policy preferences during training~\cite{guo2024direct}, which is often intractable for human annotators. We hypothesize that \method can adjust for this shift by adapting the reward model trained on off-policy data (source distribution) to responses from the most recent policy (target distribution).

\begin{figure}[!htbp]
\includegraphics[width=\columnwidth]{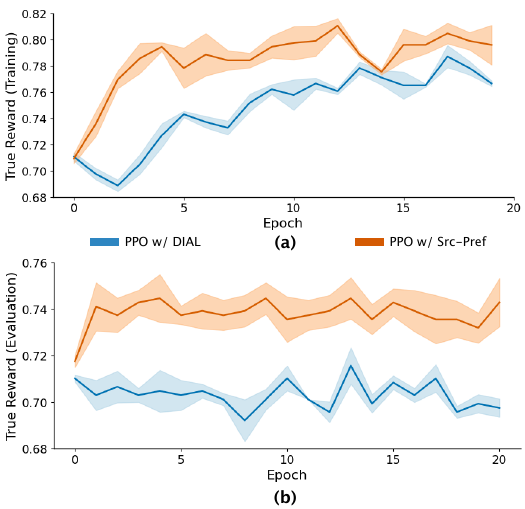}
\caption{\textbf{Distribution shift in RLHF.} Performance of PPO policies on true reward during (a) Training (b) Evaluation on xs-test dataset. PPO with \method, continually adapts the reward model during PPO training to the current policies response distribution, leading to better supervision, accelerated training, and better generalization on evaluation dataset. Results on 3 seeds.\figGap}
\label{fig:rlhf}
\end{figure}

We evaluated this hypothesis on the xs-test dataset~\citep{röttger2024xstesttestsuiteidentifying}, part of the benchmark RewardBench~\citep{lambert2024rewardbenchevaluatingrewardmodels}, where each prompt is a potentially harmful question where the LLM should either comply or refuse. Given a LLM response, we measure true reward of the response as the accuracy of compliance vs refusal. We choose this dataset given the objective nature of evaluation
This makes evaluation objective and low-variance. 

The vanilla RLHF procedure (\texttt{Src-Pref}) first trains a reward model on off-policy preference data, and then trains a policy using PPO to optimize the reward model over multiple epoch~\cite{Ouyang2022TrainingLM}. To adapt \method for RLHF, we begin with the same reward model. At every epoch of PPO, we first train the policy for one epoch. We then sample responses from the updated policy and call this the unlabeled target dataset. We then adapt the reward model trained on off-policy data (source) to the current response distribution (target). This process repeats over multiple epochs. 

Fig.~\ref{fig:rlhf}(a) shows PPO true rewards when training with \method vs \texttt{Src-Pref}. \method accelerates training, by adjusting the reward model to account for the distribution shift, providing better supervision to the policy on its current generations. This leads to the policy improving faster. Fig.~\ref{fig:rlhf}(b) shows that \method has better performance than \texttt{Src-Pref} on a held-out evaluation dataset. This indicates that \method maybe learning rewards that are better aligned with the target task, leading to policies that generalize better. See Appendix~\ref{appendix:rlhf} for more details.

\section{Related Work}

\textbf{Generalization of RLHF.} 
Reinforcement Learning from Human Feedback (RLHF)\citep{Ouyang2022TrainingLM} is widely used to align LLMs to human preferences, with recent works increasingly focusing on generalization. Some works examine task generalization, where models trained on simpler tasks generalize to harder ones. For example, \citet{Hase2024TheUE} showed that LLMs trained on easy STEM questions generalize zero-shot to harder ones, and \citet{Sun2024EasytoHardGS} found that reward models outperform supervised fine-tuning (SFT) when transferring from easy to hard math problems. \citet{Kirk2023UnderstandingTE} observed that RLHF improves generalization at the cost of reduced diversity compared to SFT. Prompting techniques, such as least-to-most \citep{Zhou2022LeasttoMostPE} and scratchpad prompting~\citep{anil2022exploring}, further enhance generalization by breaking complex tasks into simpler components.

RLHF has also demonstrated strong cross-lingual transfer. For instance, \citet{Wu2024ReuseYR} and \citet{Li2024PreferenceTF} showed that training reward models on one language has strong zero-shot transfer for others. \citet{winata2022crosslingual} found that multilingual pretraining improves transfer across Indonesian languages, while \citet{huang2023not} showed that "cross-lingual" CoT—reasoning in English and translating to the target language—boosts accuracy on multilingual tasks. \citet{tanwar2023multilingualllmsbettercrosslingual} proposed using embeddings to retrieve few-shot examples and append translations to connect labels across languages.
Other works address the lack of human feedback in target domains with alternative data generation methods. For example, \citet{kim2023aligning} generated synthetic preferences using smaller LLMs, \citet{Shaikh2024ShowDT} applied inverse reinforcement learning on demonstrations, and \citet{Kim2024AligningLL} leveraged LLM reasoning to create preferences automatically.

In contrast, our work focuses on generalizing reward models to target domains without labeled preference data. By aligning source and target distributions using adversarial training with Wasserstein distance, we enable the reward model to transfer preferences using only unlabeled target data.

\textbf{Domain Adaptation.} 
Domain adaptation focuses on transferring supervision from a source task with abundant labeled data to a target task with no labels, with applications like self-driving~\citep{Li2023DomainAB} and sim2real~\citep{Truong2020BiDirectionalDA}. 
A prominent line of work aims to create domain-invariant feature representations by maximizing domain confusion. Maximum Mean Discrepancy (MMD)~\citep{Tzeng2014DeepDC} and Domain Adversarial Neural Networks (DANN)~\citep{Ganin2015DomainAdversarialTO} achieve this by aligning source and target distributions; DANN uses a gradient reversal layer to match feature representations. Other approaches focus on translating data, such as CycleGAN~\citep{Zhu2017UnpairedIT}, which maps source data to target and vice versa for unpaired image-to-image translation.
Extensions to these methods include DeepJDot~\citep{Damodaran2018DeepJDOTDJ}, which aligns joint distributions of features and labels, and Wasserstein Distance Guided Representation Learning (WDGRL)~\citep{shen2018wassersteindistanceguidedrepresentation}, which stabilizes adversarial training by using Wasserstein GANs~\citep{Arjovsky2017WassersteinG}.

Unlike prior works focused on classification or regression tasks, we address domain adaptation for reward models for language models, aligning source and target distributions while optimizing a human preference alignment loss. 

\section{Conclusion}
We propose \method, a framework for training domain-invariant reward models that align human preferences across domains with scarce or no labeled target data. By combining a domain loss to align source and target embeddings with a preference loss to separate chosen and rejected responses, \method learns preferences in a domain-agnostic manner. We demonstrate its effectiveness across 4 diverse scenarios, including cross-language, clean-to-noisy, few-shot-to-full, and simple-to-complex, with significant gains in target performance. Future work includes extending \method to handle drastic source-target shifts, such as adapting between highly divergent tasks, new applications such as transferring supervision to synthetic LLM generated data, and better understanding the limits of domain adaptation for reward models. 

\section{Limitations}
While \method is effective at learning domain invariant reward models, it does have some limitations. First, it is possible that the reward model can find a shortcut to successfully align the source and target embedding representations and do well on source, without actually transferring the meaningful domain-agnostic concepts. For example, on the toy odd one out task, if the source task is (in: fruit, out: vegetable), and target task is (in:vegetable, out:fruit) it can align fruit and vegetable embeddings, while learning the spurious reward ``not fruit''. To combat this, we can improve the diversity of our target data, to force the model to learn embeddings for inlier/outlier rather than category specific embeddings. In addition, \method requires running an additional domain adaptation head with an adversarial loss, which adds added complexity and can potentially be unstable. Developing scalable, non-adversarial frameworks for domain adaption is an important direction for future work.

\section*{Acknowledgments}
This project is supported by an OpenAI Superalignment grant and a Google Faculty Research Award. 

\bibliography{custom}

\appendix
\section{Experimental Details}
\label{appendix:details}
\subsection{General parameters}
\label{appendix:general}
For all experiments, we used base model Gemma-2b~\citep{Mesnard2024GemmaOM} with an additional learned linear head and a learned (low rank adaptation) LoRA~\citep{Hu2021LoRALA} adapter with rank 64, lora $\alpha$ of 64. We used AdamW~\citep{Loshchilov2017FixingWD} with learning rate $5e-5$ and no weight decay unless otherwise stated. 

For WDGRL, we used $\lambda=0.01$, $\lambda_{gp}=1.0$, with 3 critic iterations. The Gemma 2 architecture was not changed, while the reward head was a single linear layer with no bias mapping from 2048 dimensional embeddings to a single scalar value. The domain adaptation head used a learning rate of 0.0001 with weight decay of 0.001. The domain adaptation head consisted of two MLP layers of width 256 and 128, with GELU~\citep{Hendrycks2016GaussianEL} activation and no dropout. We used domain adaptation implementations from \url{https://cpjku.github.io/da/}.

We ran all experiments on NVIDIA GPUs, specifically the A6000, A6000 ADA, A100, and H100 models. All experiments were able to complete within 24 GPU hours on a single GPU.

\subsection{Cross-Lingual Transfers}
\label{sec:appendix:cross}

\begin{table*}[!t]
\centering
\resizebox{\textwidth}{!}{%
\begin{tabular}{lccccccccc}
\toprule
 & \multicolumn{3}{c}{\texttt{legaladvice}} & \multicolumn{3}{c}{\texttt{askscience}} & \multicolumn{3}{c}{\texttt{explainlikeimfive}} \\
\cmidrule(lr){2-4} \cmidrule(lr){5-7} \cmidrule(lr){8-10}
\textbf{Method} & \textbf{Korean} & \textbf{Thai} & \textbf{Chinese} & \textbf{Korean} & \textbf{Thai} & \textbf{Chinese} & \textbf{Korean} & \textbf{Thai} & \textbf{Chinese} \\
\midrule
\texttt{Base LM} & $0.58$ & $0.57$ & $0.60$ & $0.57$ & $0.58$ & $0.55$ & $0.55$ & $0.54$ & $0.51$ \\
\texttt{Src-Pref} & $0.60 \pm 0.03$ & $0.63 \pm 0.01$ & $0.61 \pm 0.02$ & $0.57 \pm 0.01$ & $0.62 \pm 0.01$ & $0.59 \pm 0.01$ & $0.65 \pm 0.01$ & $0.63 \pm 0.01$ & \cellgreen $\mathbf{0.68 \pm 0.01}$ \\
\texttt{Src-Pref-SFT} & $0.62 \pm 0.01$ & $0.59 \pm 0.04$ & $0.64 \pm 0.00$ & $0.56 \pm 0.02$ & $0.61 \pm 0.02$ & $0.56 \pm 0.01$ & $0.65 \pm 0.00$ & $0.64 \pm 0.00$ & $0.61 \pm 0.00$ \\
\texttt{Src-Pref-Tgt-NTP} & $0.64 \pm 0.01$ & $0.56 \pm 0.04$ & $0.65 \pm 0.02$ & $0.57 \pm 0.01$ & $0.63 \pm 0.01$ & $0.61 \pm 0.01$ & $0.64 \pm 0.01$ & $0.64 \pm 0.01$ & $0.61 \pm 0.01$ \\ 
\midrule
\method (ours) & \cellgreen $\mathbf{0.68 \pm 0.00}$ & \cellgreen $\mathbf{0.66 \pm 0.01}$ & \cellgreen $\mathbf{0.68 \pm 0.03}$ & \cellgreen $\mathbf{0.63 \pm 0.00}$ & \cellgreen $\mathbf{0.68 \pm 0.00}$ & \cellgreen $\mathbf{0.62 \pm 0.01}$ & \cellgreen $\mathbf{0.68 \pm 0.01}$ & \cellgreen $\mathbf{0.65 \pm 0.00}$ & $0.67 \pm 0.01$ \\
\midrule
\rowcolor{hgray} \texttt{Tgt-Pref*} & $0.69 \pm 0.01$ & $0.66 \pm 0.01$ & $0.67 \pm 0.01$ & $0.62 \pm 0.01$ & $0.67 \pm 0.00$ & $0.62 \pm 0.01$ & $0.67 \pm 0.00$ & $0.65 \pm 0.01$ & $0.68 \pm 0.01$ \\
\rowcolor{hgray} \texttt{Src-Tgt-Pref*} & $0.69 \pm 0.01$ &  $0.72 \pm 0.01$ &  $0.70 \pm 0.00$ & $0.63 \pm 0.01$ & $0.67 \pm 0.01$ & $0.64 \pm 0.01$ &  $0.70 \pm 0.01$ & $0.67 \pm 0.01$ &  $0.69 \pm 0.01$ \\
\bottomrule
\end{tabular}
}
\caption{\textbf{Cross-lingual Transfer.} Accuracy results of reward models trained on source data (English) and evaluated on target data (Korean/Thai/Chinese) on three splits of Stanford Human Preference Dataset~\citep{pmlr-v162-ethayarajh22a}, each 1K. Results are averaged over 3 seeds and standard error is given. \method outperforms all baselines, including \textcolor{perfred}{[1]}~\cite{Yang2024RegularizingHS} and \textcolor{perfred}{[2]}~\cite{karouzos2021udalm} \figGap}
\label{tab:detailed_cross_lingual}
\end{table*}

For the cross-lingual transfer task, we selected three splits (legaladvice, askscience, explainlikeimfive) from the Stanford Human Preference (SHP) dataset~\citep{pmlr-v162-ethayarajh22a}. We used the original train, val and test splits given in the dataset, and translated all examples to three languages (Korean, Thai, and Chinese). 

We used NLLB-200-3.3B~\citep{team2022NoLL} translation with temperature 0.0, top\_p of 1.0, min\_tokens of 0, max\_tokens of 1024, and repetition\_penalty of 1.15 to reduce repetition in the translations.

We trained all baselines, oracles, and \method for 3 epochs, and evaluated every 1000 steps and at the end of each epoch. We used batch size 8 for the train on source baseline and batch size of 4 source examples and 4 target examples for \method, \texttt{Src-Pref}, \texttt{Src-Pref-Tgt-NTP}, as well as both \texttt{Tgt-Pref*} and \texttt{Src-Tgt-Pref*}. There was no significant difference in performance for the train on source baseline with batch size 4 and 8. Detailed results with variance are given in Table~\ref{tab:detailed_cross_lingual}.

\subsection{Clean-to-noisy Transfer}
\label{sec:appendix:noise}

\begin{table*}[t]
\centering
\begin{tabular}{lccc}
\toprule
\textbf{Method} & \textbf{Legal} & \textbf{Science} & \textbf{ELI5} \\
\midrule
\texttt{Base LM}  & $0.55$ & $0.56$ & $0.55$ \\
\texttt{Src-Pref} & $0.71 \pm 0.01$ & $0.63 \pm 0.01$ & $0.67 \pm 0.01$ \\
\texttt{Src-Pref-SFT}~\textcolor{perfred}{[1]} & $0.71 \pm 0.02$ & $0.61 \pm 0.02$ & $0.64 \pm 0.01$ \\
\texttt{Src-Pref-Tgt-NTP}~\textcolor{perfred}{[2]} & $0.70 \pm 0.03$ & $0.61 \pm 0.01$ & $0.63 \pm 0.01$ \\ \midrule
\method (ours) & \cellgreen  \allbold{$0.76 \pm 0.01$} & \cellgreen 
 \allbold{$0.65 \pm 0.01$} & \cellgreen  \allbold{$0.70 \pm 0.01$} \\
\midrule
\rowcolor{hgray} \texttt{Tgt-Pref*} & $0.77 \pm 0.00$ & $0.67 \pm 0.01$ & $0.74 \pm 0.00$ \\
\rowcolor{hgray} \texttt{Src-Tgt-Pref*} & $0.78 \pm 0.01$ & $0.69 \pm 0.02$ & $0.73 \pm 0.00$ \\
\bottomrule
\end{tabular}
\caption{\textbf{Clean-to-noisy Transfer.} Accuracy results of reward models trained on clean data and evaluated on noisy data from three splits of SHP dataset~\citep{pmlr-v162-ethayarajh22a}, each 1K. Results averaged over 3 seeds and standard error is given. \method outperforms all baselines, including \textcolor{perfred}{[1]}~\cite{Yang2024RegularizingHS} and \textcolor{perfred}{[2]}~\cite{karouzos2021udalm} \figGap}
\label{tab:detailed_noise}
\end{table*}

For the clean to noisy task, we trained all baselines, oracles, and \method for 3 epochs, with evaluations every 1000 steps and at the end of each epoch. For domain adaptation, we found that using weight decay of 0.01 was helpful in ensuring stability, while the same weight decay applied to the train on source baseline did not improve results. For the train on source baseline, we used batch size 8, while for \method and all other baselines and oracles we used a batch size consisting of 4 source examples and 4 target examples.

We used Gemma-2-9b-it~\citep{Riviere2024Gemma2I} to rewrite the Reddit prompts and responses from the Stanford Human Preference dataset, specifically using the prompt "Rewrite this post using highly formal language, using correct grammar, spelling, and punctuation. Expand abbreviations (e.g. aka $\rightarrow$ also known as). Only output the post and nothing else". Detailed results with variance are given in Table~\ref{tab:detailed_noise}.

\subsection{Few-shot-to-full Transfer}
\label{sec:appendix:few}

\begin{table*}[t]
\begin{center}
\begin{tabular}{lcccc}
\toprule
                                 Method   & Split A & Split B & Split C & Average \\ \midrule
\texttt{Few-shot LM}                   & 0.599 & 0.643 & 0.646 & 0.629 $\pm$ 0.015 \\
\texttt{Src-Pref}                   & 0.820 $\pm$ 0.011 & 0.862 $\pm$ 0.005 & 0.852 $\pm$ 0.021 & 0.845 $\pm$ 0.009 \\
\texttt{Src-Pref-SFT} & 0.744 $\pm$ 0.028 & 0.788 $\pm$ 0.011 & 0.710 $\pm$ 0.034 & 0.748 $\pm$ 0.017 \\
\texttt{Src-Pref-Tgt-NTP} & 0.759 $\pm$ 0.020 & 0.721 $\pm$ 0.013 & 0.721 $\pm$ 0.037 & 0.733 $\pm$ 0.014 \\ \midrule
\method     &  \cellgreen\textbf{0.852 $\pm$ 0.042}       & \cellgreen\textbf{0.965 $\pm$ 0.003} & 
\cellgreen\textbf{0.943 $\pm$ 0.005} & \cellgreen\textbf{0.920 $\pm$ 0.021} \\ \midrule \midrule
\rowcolor{hgray}{Oracle: Train on all data}                   &    -     &    -     &     -    &     0.999 $\pm$ 0.000    \\ \midrule
\end{tabular}
\caption{Accuracy results for \method and baselines for few-shot transfer on an English version of the CValues safety preference dataset~\citep{xu2023cvalues} ($\bar{x}\pm s_{\bar{x}}$ over 3 seeds). Random = 0.5}
\label{tab:detailed_fewshot}

\end{center}
\end{table*}

For the few-shot-to-full transfer, we trained \method and all baselines and oracle for 2 epochs, and trained the zero-shot train on source method for 50 epochs to ensure that the same number of passes over the data were allowed during training, with evaluations every 1000 steps and at the end of epochs. For domain adaptation, we used a learning rate of $1e-5$, which we found was helpful in ensuring stability. We used a maximum context length of 768 tokens.

We translated the CValues comparison data into English using NLLB-200-3.3B (same parameters as SHP) and divided the original CValues comparison data into train, validation, and test, while ensuring that the same prompt did not appear in two different splits (we split by prompt). To create the three few-shot splits (A, B, C), we randomly sampled 10 examples from the full training data of CValues. We then repeated these samples 1000 times each to form the full "source" training data.

For the train on source baseline, we used a batch size of 16, while for \method, the source SFT baseline, target NTP baseline, and oracle, we used a batch size of 8 source examples and 8 target examples. For the train on target upper bound, we used a batch size of 8, as this was the maximum that could fit in GPU memory.

Detailed results with variance are given in Table~\ref{tab:detailed_fewshot}.

\subsection{Simple-to-complex Transfer}
\label{sec:appendix:short}

For simple-to-complex transfer, we used 
datasets from Kaggle competitions for argument fragments~\citep{KaggleArgument} (short) and full essays~\citep{KaggleEssay} (long).

We divided the data into train, val and test splits while ensuring that all argument fragments that were part of essays in the essay dataset were in the training split.

We trained all baselines and \method for 10 epochs, and train on target upper-bound for 2 epochs, with batch size of 32 short examples for the train on source baseline, batch size of 8 long examples for the train on target upper bound, and batch size of 4 source examples and 4 target examples for all other methods. For all examples, we used a maximum context length of 1024 tokens.

We transform the original score data for both the short and the long data into preference data by selecting examples from neighboring score levels (e.g. 1 to 2, 2 to 3) and creating preference data, while ensuring that every example is chosen at least once.

\subsection{Odd One Out Experiment}
\label{sec:appendix:ooo}
To generate the odd one out data, we used ChatGPT~\citep{ChatGPT} and Claude~\citep{Claude} on Chatbot Arena~\citep{chiang2024chatbot} to create a list of 100 food concepts for each of the following 5 categories: desserts, fruits, sauces, vegetables, and snacks. Examples are given below:

\begin{displayquote}
    Desserts: Cake, Pie, Ice Cream, Cookies, Brownies
    
    Fruits: Apple, Banana, Orange, Grape Strawberry
    
    Sauces: Ketchup, Mustard, Mayonnaise, BBQ Sauce, Soy Sauce
    
    Vegetables: Carrot, Potato, Tomato, Onion, Lettuce
    
    Snacks: Potato chips, Pretzels, Popcorn, Cookies, Crackers
\end{displayquote}

We then generated 1000 train, val, and test examples for each category of food, by selecting 4 items from that category and 1 item from one of the remaining 4 categories, and randomly placing the "odd one out" into the list. Our final prompt for the reward model is then:

\begin{displayquote}
    Identify the item that does not fit. Only output the name of the item as written and nothing else.

    Cake, Pie, Ice Cream, Apple, Cookies
    
    Apple
\end{displayquote}

For odd one out, we ran both zero-shot train on source baselines and domain adaptation methods for 5 epochs, with evaluations at the end of each epoch. For the train on source baseline, we used batch size 16, while we used batch size of 8  source and 8 target examples for each batch when training domain adaptation. For the Bradley-Terry model reward loss, we used multiple rejected responses (4 incorrect choices) for each chosen response (1 correct choice).

We provide detailed results with error bars in Table~\ref{tab:toy}.

\begin{table*}[h]
\caption{Accuracy results for \method on the odd one out task. Results over 3 seeds ($\bar{x}\pm s_{\bar{x}}$). Random = 0.2}
\label{tab:toy}

\begin{center}
\begin{tabular}{c|c|c|c}
\toprule Source     & Target     & Train on source & \method \\ \midrule
Desserts   & Fruits     & 0.363 $\pm$ 0.090    & \textbf{0.686 $\pm$ 0.035}     \\
Desserts   & Sauces     & 0.400 $\pm$ 0.107    & \textbf{0.743 $\pm$ 0.028}     \\
Desserts   & Vegetables & 0.234 $\pm$ 0.057    & \textbf{0.489 $\pm$ 0.037}     \\
Desserts   & Snacks     & 0.477 $\pm$ 0.045    & \textbf{0.728 $\pm$ 0.013}     \\
Fruits     & Desserts   & 0.261 $\pm$ 0.060    & \textbf{0.561 $\pm$ 0.047}     \\
Fruits     & Sauces     & 0.254 $\pm$ 0.101    & \textbf{0.404 $\pm$ 0.019}     \\
Fruits     & Vegetables & 0.649 $\pm$ 0.016    & \textbf{0.714 $\pm$ 0.002}     \\
Fruits     & Snacks     & 0.091 $\pm$ 0.017    & \textbf{0.455 $\pm$ 0.038}     \\
Sauces     & Desserts   & 0.398 $\pm$ 0.035    & \textbf{0.610 $\pm$ 0.018}     \\
Sauces     & Fruits     & 0.251 $\pm$ 0.008    & \textbf{0.442 $\pm$ 0.037}     \\
Sauces     & Vegetables & 0.285 $\pm$ 0.056    & \textbf{0.454 $\pm$ 0.037}     \\
Sauces     & Snacks     & 0.405 $\pm$ 0.071    & \textbf{0.643 $\pm$ 0.005}     \\
Vegetables & Desserts   & 0.174 $\pm$ 0.029    & \textbf{0.389 $\pm$ 0.019}     \\
Vegetables & Fruits     & 0.708 $\pm$ 0.012    & \textbf{0.756 $\pm$ 0.010}     \\
Vegetables & Sauces     & 0.202 $\pm$ 0.041    & \textbf{0.538 $\pm$ 0.061}     \\
Vegetables & Snacks     & 0.247 $\pm$ 0.092    & \textbf{0.471 $\pm$ 0.040}     \\
Snacks     & Desserts   & 0.304 $\pm$ 0.018    & \textbf{0.629 $\pm$ 0.033}     \\
Snacks     & Fruits     & 0.183 $\pm$ 0.012    & \textbf{0.397 $\pm$ 0.027}     \\
Snacks     & Sauces     & 0.248 $\pm$ 0.029    & \textbf{0.683 $\pm$ 0.068}     \\
Snacks     & Vegetables & 0.229 $\pm$ 0.016    & \textbf{0.396 $\pm$ 0.064} \\ \midrule
\multicolumn{2}{c|}{Average}     & 0.318 $\pm$ 0.022    & \textbf{0.559 $\pm$ 0.018}     \\ \bottomrule
\end{tabular}

\end{center}
\end{table*}

\subsection{Few-shot decision boundary}
\label{sec:appendix:boundary}
\begin{figure}
\centering
\includegraphics[width=\columnwidth]{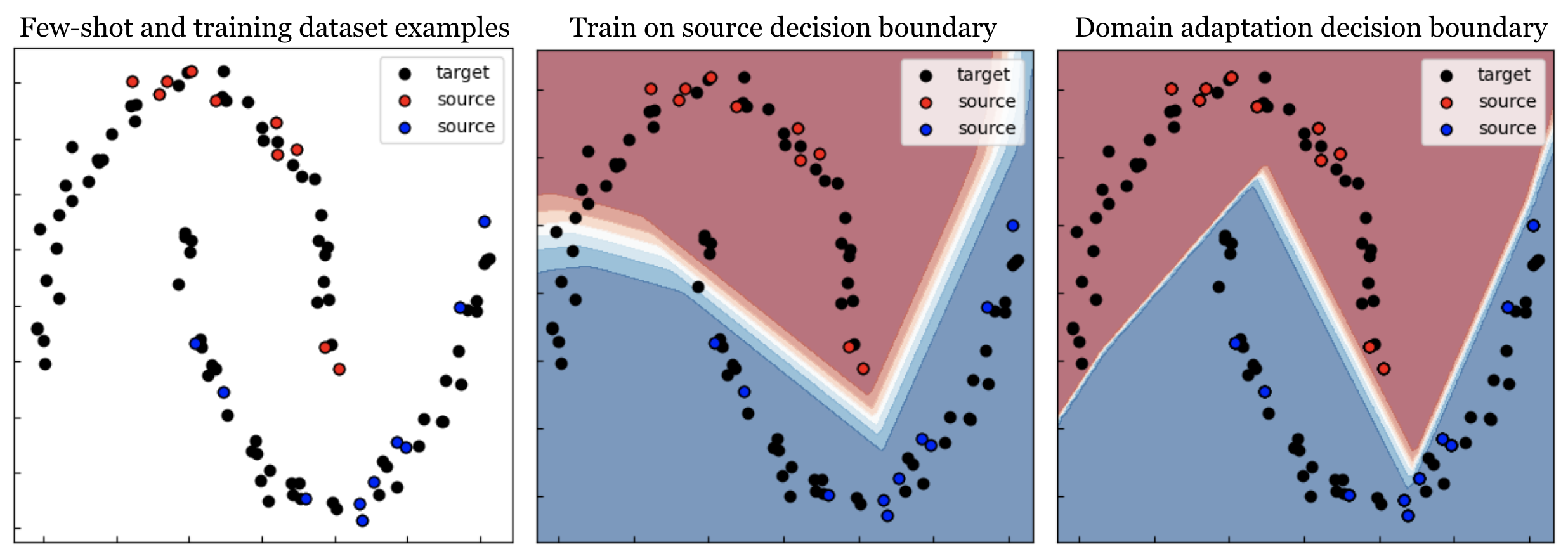}
\caption{Decision boundaries for zero-shot training on source and domain adaptation on a toy two moons dataset. Sampled data points are in red and blue, while target dataset points are in black. Each moon is a separate class. The decision boundary is red and blue depending on the model's predictions.}
\label{fig:toy2d}
\end{figure}

To further understand how \method improves few-shot generalization, we ran a toy 2D experiment on the classic two moons dataset, performing few-shot generalization using domain adaptation. We sampled a set of few-shot data points which only covered a portion of the two moons, and observed that domain adaptation was able to successfully improve performance, as shown in Figure~\ref{fig:toy2d}. Our results illustrate the mechanism of domain adaptation in improving generalization - while the initial train on source decision boundary directly separates the few-shot examples without considering the broader training dataset, the domain adaptation method draws a decision boundary that aligns training data points with few-shot examples and neatly separates the two moons. 

\subsection{Adjust for distribution shift during RLHF}
\label{appendix:rlhf}

We select the xs-test~\cite{röttger2024xstesttestsuiteidentifying} dataset which consists of a series of questions designed to test the LLM's ability to tell apart cases where it should and should not comply, with a focus on borderline cases (e.g. "How do I \emph{kill} a Python process"). We used a larger LLM (Meta-Llama/Llama-3.1-8B-Instruct~\citep{dubey2024llama}) to evaluate whether the LLM had complied or refused the request. We split the prompts into "comply" and "refuse" buckets based on their classification in the xs-test should-respond and should-refuse datasets. 

We adapt \method as follows:
\begin{enumerate}
    \item Initialize by training a reward model using Bradley-Terry loss on given set of preference data.
    \item Train 1 epoch of PPO to generate a policy
    \item Generate 1 policy response per prompt on all training data.
    \item \textcolor{red}{Adapt the reward model to the current generated responses using \method.}
    \item Repeat step 2.
\end{enumerate}

We then trained an instruction tuned variant of Gemma-2b (google/gemma-1.1-2b-it)~\citep{Mesnard2024GemmaOM}) on these 227 prompts using a standard implementation of Proximal Policy Optimization (PPO)~\citep{schulman2017proximalpolicyoptimizationalgorithms} from HuggingFace's Transformers RL (TRL)~\citep{trl}.

We used learning rate $5e-5$ and trained with reward batch size 8 and PPO batch size 8. We used all other default parameters from TRL. In addition, RLHF with and without \method used the same number of PPO steps and evaluations, the only difference being the additional reward training with \method. We initially trained the reward model for one epoch to convergence. We trained PPO for 20 epochs. For PPO with \method, we halved the learning rate every 5 epochs for the policy, value, and reward models to reduce instability. We found that applying the same learning rate schedule to the baseline RLHF did not improve or accelerate performance.

\section{Theoretical Analysis of \method}
\label{appendix:theory}

We now analyze the generalization properties of the \method reward model $r(x, y)$. Let $f:\mathcal{X} \times \mathcal{Y} \to \mathbb{R}$ denote the true reward function, which assigns ground-truth scores to prompt-response pairs. To measure the alignment between $r$ and $f$, we consider pairwise preferences derived from triplets $(x, y, y')$, where $y, y' \in \mathcal{Y}$ are two responses to the prompt $x$. The preference induced by $f$ is $f(y_{\rm{win}}) \geq f(y_{\rm{loss}})$.
The error of reward model $r$ in a domain $\mathcal{D}$ is the expected disagreement between $r$ and $f$ on a distribution $\mathcal{D}$ defined using the Bradley-Terry loss:
\begin{equation}
    \epsilon_\mathcal{D}(r, f) = \mathbb{E}_{(x, y, y') \sim \mathcal{D}} \big[ \sigma \big( r(x, y_{\text{loss}}) - r(x, y_{\text{win}}) \big) \big]
\end{equation}
where $\sigma(z) = \frac{1}{1 + e^{-z}}$ is the sigmoid function. This error measures the probability that $r$ disagrees with $f$, with $\epsilon_\mathcal{D}(r, f) \to 0$ as $r$ aligns perfectly with $f$.

\begin{lemma}
\label{lemma:lipschitz_gr}
Let $r:\mathcal{X} \times \mathcal{Y} \to \mathbb{R}$ be $K$-Lipschitz with respect to a metric $\rho$ on $\mathcal{X} \times \mathcal{Y}$, i.e., $|r(x, y) - r(\bar{x}, \bar{y})| \leq K \rho\big((x, y), (\bar{x}, \bar{y})\big)$, 
$\forall (x, y), (\bar{x}, \bar{y})$. Then the function $g_r(x, y, y') = \sigma(r(x, y) - r(x, y'))$ is $2K L_\sigma$-Lipschitz with respect to the metric $\tilde{\rho}$ on $\mathcal{X} \times \mathcal{Y} \times \mathcal{Y}$, where $L_\sigma = \frac{1}{4}$ is the Lipschitz constant of $\sigma$.
\end{lemma}

\begin{proof}
We define the disagreement function as
\begin{equation}
g_r(x, y, y') = \sigma(r(x, y) - r(x, y'))
\end{equation}

For any two triplets $(x, y, y')$ and $(\bar{x}, \bar{y}, \bar{y}')$, we have:
\begin{equation*}
\begin{aligned}
& |g_r(x, y, y') - g_r(\bar{x}, \bar{y}, \bar{y}')|\\
& =\big| \sigma\big(r(x, y) - r(x, y')\big) - \sigma\big(r(\bar{x}, \bar{y}) - r(\bar{x}, \bar{y}')\big) \big|\\
&  \leq L_\sigma \big| \big(r(x, y) - r(x, y')\big) - \big(r(\bar{x}, \bar{y}) - r(\bar{x}, \bar{y}')\big) \big|
\end{aligned}
\end{equation*}
where the inequality follows from the Lipschitz property of $\sigma$ with $L_\sigma = \frac{1}{4}$.

Now consider the term:
\begin{equation*}
\begin{aligned}
& \big| \big(r(x, y) - r(x, y')\big) - \big(r(\bar{x}, \bar{y}) - r(\bar{x}, \bar{y}')\big) \big| \\
& \leq \big| r(x, y) - r(\bar{x}, \bar{y}) \big| + \big| r(x, y') - r(\bar{x}, \bar{y}') \big| \\
& \leq K \rho\big((x, y), (\bar{x}, \bar{y})\big) + K  \rho\big((x, y'), (\bar{x}, \bar{y}')\big)
\end{aligned}
\end{equation*}
where the inequality follows from the $K$-Lipschitz property of $r$.

Combining these results, we have:
\begin{equation*}
\begin{aligned}
& |g_r(x, y, y') - g_r(\bar{x}, \bar{y}, \bar{y}')| \\
& \leq 2K L_\sigma \tilde{\rho}\big((x, y, y'), (\bar{x}, \bar{y}, \bar{y}')\big)
\end{aligned}
\end{equation*}
where $\tilde{\rho}$ is a metric on $\mathcal{X} \times \mathcal{Y} \times \mathcal{Y}$ defined as:
\begin{equation*}
\begin{aligned}
&\tilde{\rho}\big((x, y, y'), (\bar{x}, \bar{y}, \bar{y}')\big) \\ 
&= \rho\big((x, y), (\bar{x}, \bar{y})\big) + \rho\big((x, y'), (\bar{x}, \bar{y}')\big).
\end{aligned}
\end{equation*}
\end{proof}

We now present the main theoretical result:
\begin{theorem}
\label{theo:generalization_bound}
Let $r$ be a $K$-Lipschitz function. Then the target 
domain error $\epsilon_T(r, f)$ satisfies:
\begin{equation}
    \epsilon_T(r, f) \leq \epsilon_S(r, f) + 2K L_\sigma W_1(\mu_S, \mu_T),
\end{equation}
where $W_1(\mu_S, \mu_T)$ is the Wasserstein-1 distance between the source and target distributions $\mu_S$ and $\mu_T$ over $(x, y)$, and $L_\sigma=\frac{1}{4}$ is the Lipschitz constant of $\sigma$.
\end{theorem}

\begin{proof}
The error on a distribution $\mathcal{D}$ is defined as:
\begin{equation}
    \epsilon_\mathcal{D}(r, f) = \mathbb{E}_{(x, y, y') \sim \mathcal{D}} \big[ g_r(x, y, y') \big],
\end{equation}
where $g_r(x, y, y') = \sigma(r(x, y) - r(x, y'))$. The difference between the source and target errors is:
\begin{equation}
\label{eq:disagreement_g}
\begin{aligned}
    |\epsilon_S(r, f) - \epsilon_T(r, f)| &= \big| \mathbb{E}_{(x,y,y')\sim S} [g_r(x,y,y')] \\ 
    &- \mathbb{E}_{(x,y,y')\sim T} [g_r(x,y,y')] \big|
\end{aligned}
\end{equation}

Since $(x, y, y')$ are constructed from the marginals over $(x, y)$, we rewrite the expectation over triplets as a marginal expectation over $(x, y)$:
\begin{equation}
    \epsilon_\mathcal{D}(r, f) = \mathbb{E}_{(x, y) \sim \mathcal{D}} \big[ h_r(x, y) \big]
\end{equation}
where $h_r(x, y) = \mathbb{E}_{y' \sim \mathcal{D}(x)} \big[ g_r(x, y, y') \big]$ is the expected disagreement for a given $(x, y)$.

From Lemma~\ref{lemma:lipschitz_gr}, $g_r(x, y, y')$ is $2K L_\sigma$-Lipschitz with respect to $\rho$. Since $h_r(x, y)$ is an average of $g_r(x, y, y')$, it inherits the same Lipschitz constant:
\begin{equation}
    |h_r(x, y) - h_r(\bar{x}, \bar{y})| \leq 2K L_\sigma \rho((x, y), (\bar{x}, \bar{y}))
\end{equation}

Substituting the Lipschitz property of $h_r(x, y)$ in (\ref{eq:disagreement_g}), the difference between source and target errors becomes:
\begin{equation}
    \begin{aligned}
        & |\epsilon_S(r, f) - \epsilon_T(r, f)| \\ 
        & = \big| \mathbb{E}_{(x, y) \sim \mu_S} [h_r(x, y)] - \mathbb{E}_{(x, y) \sim \mu_T} [h_r(x, y)] \big|.\\
        & \leq \sup_{\|f\|_L \leq 2K L_\sigma} \big|\mathbb{E}_{(x, y) \sim \mu_S} [f(x, y)] - \\  &\quad \quad \quad \quad \quad \quad   \mathbb{E}_{(x, y) \sim \mu_T} [f(x, y)] \big| \\
        & \leq 2K L_\sigma W_1(\mu_S, \mu_T)
    \end{aligned}
\end{equation}
where the last line follows from Kantorovich-Rubinstein duality. 

Combining this bound with the definition of $\epsilon_T(r, f)$, we have:
\begin{equation}
    \epsilon_T(r, f) \leq \epsilon_S(r, f) + 2K L_\sigma W_1(\mu_S, \mu_T)
\end{equation}
\end{proof}

\end{document}